\documentclass[review,12pt,authoryear]{elsarticle}

\usepackage[T1]{fontenc}
\usepackage[utf8]{inputenc}

\usepackage{amsmath,amssymb,amsthm,mathtools}

\usepackage{booktabs}
\usepackage{array}
\usepackage{tabularx}
\usepackage{multirow}
\usepackage{url}
\usepackage{natbib}
\usepackage{graphicx}
\usepackage{xcolor}
\usepackage{listings}
\usepackage{algorithm2e}
\usepackage{tikz}
\usepackage{pgfplots}
\usepackage{tcolorbox}

\usepackage{lineno}
\usepackage{footnote}  
\usepackage{geometry}
\usepackage{setspace}
\usepackage{enumitem}
\usepackage{subcaption}
\usepackage{mdframed}
\usepackage{hyperref}
\newcommand{\doi}[1]{\href{https://doi.org/#1}{doi:#1}}
\usepackage{cleveref}

\pgfplotsset{compat=1.18}
\usetikzlibrary{arrows.meta,positioning,shapes.geometric,fit,backgrounds,
                decorations.pathreplacing,calc}

\hypersetup{
  colorlinks=true,
  linkcolor=blue!60!black,
  citecolor=blue!60!black,
  urlcolor=blue!70!black,
  pdftitle={Attention at the Theoretical Minimum},
  pdfauthor={Mullin, Hains}
}

\definecolor{moablue}{RGB}{0,70,140}
\definecolor{moateal}{RGB}{0,130,130}
\definecolor{moagold}{RGB}{180,120,0}
\definecolor{codegreen}{RGB}{0,110,50}
\definecolor{codegray}{RGB}{100,100,100}
\definecolor{codepurple}{RGB}{100,0,130}
\definecolor{backgray}{RGB}{248,248,248}

\lstdefinestyle{pythonstyle}{
  language=Python,
  basicstyle=\ttfamily\small,
  columns=flexible,
  keepspaces=true,
  frame=single,
  rulecolor=\color{gray!50},
  backgroundcolor=\color{backgray},
  breaklines=true,
  numbers=left,
  numberstyle=\tiny\color{codegray},
  numbersep=8pt,
  xleftmargin=12pt,
  keywordstyle=\color{moablue}\bfseries,
  commentstyle=\color{codegreen}\itshape,
  stringstyle=\color{codepurple},
  showstringspaces=false,
  tabsize=4,
}
\tcbuselibrary{theorems,skins,breakable}

\newtcbtheorem[number within=section]{moatheorem}{Theorem}{%
  enhanced, breakable,
  colback=moablue!5, colframe=moablue!60!black,
  fonttitle=\bfseries, separator sign={.\ },
  attach boxed title to top left={yshift=-2mm,xshift=4mm},
  boxed title style={colback=moablue!70!black}
}{thm}

\newtcbtheorem[number within=section]{moaprop}{Proposition}{%
  enhanced, breakable,
  colback=moateal!5, colframe=moateal!70!black,
  fonttitle=\bfseries, separator sign={.\ },
  attach boxed title to top left={yshift=-2mm,xshift=4mm},
  boxed title style={colback=moateal!70!black}
}{prop}

\newtcbtheorem[number within=section]{moadef}{Definition}{%
  enhanced, breakable,
  colback=gray!5, colframe=gray!60,
  fonttitle=\bfseries, separator sign={.\ },
  attach boxed title to top left={yshift=-2mm,xshift=4mm},
  boxed title style={colback=gray!50}
}{def}

\newtcolorbox{keyresult}{%
  enhanced, breakable,
  colback=moagold!8, colframe=moagold!80!black,
  title={\bfseries Key Result},
  fonttitle=\bfseries,
  attach boxed title to top left={yshift=-2mm,xshift=4mm},
  boxed title style={colback=moagold!80!black}
}

\newtheorem{remark}{Remark}[section]
\newtheorem{example}{Example}[section]

\newcommand{\moaPsi}[2]{#1 \,\psi\,#2}
\newcommand{\moaPsiEnum}[2]{\langle #1 \rangle\,\psi\,#2}
\newcommand{\moaRho}[1]{\rho\,#1}
\newcommand{\moaRav}[1]{\mathrm{rav}\,#1}
\newcommand{\moaIota}[1]{\iota\,#1}
\newcommand{\moaGamma}[2]{\gamma(#1,\,#2)}
\newcommand{\moaOmega}[1]{\Omega\langle #1 \rangle}
\newcommand{\redplus}[1]{(\mathrm{red}_{+}\;{#1})}
\newcommand{\redplusV}[1]{\mathrm{red}_{+}\!\left(#1\right)}
\newcommand{\redtimes}[1]{(\mathrm{red}_{\times}\;{#1})}
\newcommand{\ceilred}{\lceil\mathrm{red}\rceil}
\newcommand{\shape}[1]{\langle #1 \rangle}

\newcommand{\concat}{\mathbin{\#}}
\newcommand{\R}{\mathbb{R}}
\newcommand{\Qm}{\mathbf{Q}}
\newcommand{\Km}{\mathbf{K}}
\newcommand{\Vm}{\mathbf{V}}
\newcommand{\dk}{d_k}
\newcommand{\dv}{d_v}
\newcommand{\softmax}{\mathrm{softmax}}
\newcommand{\attn}{\mathrm{Attention}}
\newcommand{\rank}[1]{\delta #1}
\newcommand{\drop}[2]{#1\!\downarrow\!#2}   
\newcommand{\taucomp}[1]{\tau(#1)}             
\newcommand{\DNF}{\text{DNF}}
\newcommand{\ONF}{\text{ONF}}
\newcommand{\MoA}{\textsc{MoA}}
\newcommand{\bigO}[1]{O\!\left(#1\right)}

\crefname{tcb@cnt@moatheorem}{Theorem}{Theorems}
\Crefname{tcb@cnt@moatheorem}{Theorem}{Theorems}
\crefname{tcb@cnt@moaprop}{Proposition}{Propositions}
\Crefname{tcb@cnt@moaprop}{Proposition}{Propositions}
\crefname{tcb@cnt@moadef}{Definition}{Definitions}
\Crefname{tcb@cnt@moadef}{Definition}{Definitions}
\begin{document}

\begin{frontmatter}

\title{Attention at the Theoretical Minimum:\\
  A Mathematics of Arrays Framework for\\
  Memory-Optimal Transformer Kernels}

\author[albany]{Lenore Mullin\corref{cor}}
\ead{lmullin@albany.edu}
\author[paris]{Ga\'{e}tan Hains}
\ead{gaetan.hains@u-pec.fr}

\cortext[cor]{Corresponding author}
\affiliation[albany]{organization={University at Albany, SUNY},
  city={Albany}, state={NY}, country={USA}}
\affiliation[paris]{organization={LACL, Universit\'{e} Paris-Est Cr\'{e}teil},
  city={Cr\'{e}teil}, country={France}}

\begin{abstract}
The attention mechanism is the computational core of modern transformer-based
artificial intelligence systems---the primitive that allows models to dynamically
determine which parts of an input sequence are most relevant when producing each
output. Yet its standard implementation incurs \emph{quadratic memory traffic} in
the sequence length, and real-world performance is dominated by memory-bound
execution rather than arithmetic throughput. A DRAM access costs 100--1000$\times$
more energy than a floating-point operation on contemporary hardware, so any
analysis that focuses solely on FLOP counts fundamentally mischaracterises the
bottleneck.

This paper presents a \emph{Mathematics of Arrays} (\MoA{}) reformulation of
scaled dot-product attention and its numerically stable softmax normalization. We
derive a \emph{Denotational Normal Form} (\DNF{}) that eliminates all intermediate
arrays---including the implicit transposed-key buffer and every softmax
temporary---and achieves the theoretical minimum memory traffic by algebraic
construction rather than empirical tuning. The derivation proceeds by systematic
application of \MoA{}'s psi-reduction calculus to a PyTorch reference
implementation, and is verified numerically at full double-precision floating-point
accuracy on concrete inputs.

We position this work relative to the state of the art in attention optimization,
arguing that hardware-specific accelerators, algorithmic approximations such as
FlashAttention, and array-level tiling techniques all lack a portable algebraic
foundation capable of \emph{simultaneously} providing fusion, shape-transformation
correctness, and predictive cost models. \MoA{} provides all three. A formal lower
bound argument establishes that the \DNF{} achieves $\bigO{n\dk + n\dv}$ data
movement---where $n$ is the sequence length, $\dk$ the key dimensionality, and
$\dv$ the value dimensionality (number of columns in the value matrix
$\mathbf{V}\in\mathbb{R}^{n\times\dv}$, i.e.\ the output token
representation dimension)---matching the information-theoretic minimum
of $\bigO{n\dk + n\dv}$, compared to
$\bigO{n^2 + n\dk + n\dv}$ for the standard implementation. A predictive
performance model, grounded in the literature on memory hierarchy costs, projects
$2$--$100\times$ speedup and $2$--$50\times$ energy reduction across conservative,
realistic, and aggressive deployment scenarios, with the advantage widening
significantly at exascale sequence lengths.

The derivation establishes a mathematically verifiable pipeline from Python source
program through architecture-independent \DNF{}, Operational Normal Form (\ONF{}),
and dimension-lifted hardware mapping---providing a rigorous foundation for
formally correct, performance-portable AI kernels of direct relevance to DARPA
edge-deployment objectives and DOE exascale scientific computing priorities.
\end{abstract}


\end{frontmatter}

\section{Introduction}\label{sec:intro}

The \emph{attention mechanism}~\citep{vaswani2017attention} is the computational
primitive that allows modern transformer-based neural networks to dynamically
determine which parts of an input sequence are most relevant when producing each
output token. Every state-of-the-art large language model, vision transformer, and
multimodal system in widespread use today is built around this operation. Yet its
power comes at a fundamental cost: the standard implementation of scaled
dot-product attention exhibits \emph{quadratic data movement} in the sequence
length~$n$ and is dominated by memory-bound execution rather than arithmetic
throughput.

\subsection{The Memory Bottleneck}

A simplified execution model captures the issue precisely. Let $T$ denote total
wall-clock time, $F$ the floating-point operation count, $M$ the bytes of memory
traffic, and $\alpha$, $\beta$ the cost-per-FLOP and cost-per-byte-moved
respectively. Then:
\begin{equation}\label{eq:cost_model}
  T \;\approx\; \alpha \cdot F \;+\; \beta \cdot M,
  \qquad \text{where } \beta \gg \alpha.
\end{equation}
On contemporary GPU hardware a DRAM access costs approximately $100$--$1000$\,pJ
while a floating-point multiply-add costs $\sim\!1$\,pJ~\citep{dao2022flashattention}.
Memory movement is therefore two to three orders of magnitude more expensive than
arithmetic. Any optimization strategy that focuses solely on FLOP reduction is
misallocating effort. The attention mechanism, with its $O(n^2)$ intermediate score
matrix that must be written to DRAM and re-read for the softmax pass, sits squarely
in this trap.

\subsection{The Gap in Existing Approaches}

The existing landscape of attention optimization can be organized into three
families, each of which addresses part of the problem but leaves a fundamental gap.

\textbf{Hardware-specific accelerators.} A substantial body of work targets
attention acceleration through custom silicon. FACT~\citep{qin2023fact} combines
custom hardware with diagonal storage and shift-based arithmetic. OPTIMUS
\citep{park2020optimus} overlaps load and compute while reducing redundancy via
sparse formats. The Gemmini systolic array accelerator~\citep{sharma2025gemmini}
applies GEMM acceleration to transformer kernels with convincing energy savings.
FPGA-based approaches~\citep{leonvega2023strassen,li2025fpga,tan2025fpga} achieve
high throughput on specific targets. Edge-device accelerators~\citep{chang2025mosa}
target mobile deployment. The shared limitation is that gains are inseparable from
the specific hardware platform: they cannot be formally reasoned about, ported, or
proved correct on new architectures.

\textbf{Hybrid and algorithmic solutions.} FlashAttention~\citep{dao2022flashattention}
and its successor~\citep{dao2023flashattention2} tile the attention computation to
avoid writing the full $n\!\times\!n$ score matrix to DRAM, achieving IO-awareness
through careful blocking with a proven IO complexity bound. The A$^3$
architecture~\citep{ham2020a3} replaces the scoring dot product with an
approximate content-based search at $O(n/\log n)$ complexity. Yang et
al.~\citep{yang2025hardware} exploit singular-value sparsity and approximate
softmax via a MacLaurin series expansion. Gao et al.~\citep{gao2023fast} reformulate
LLM attention as a single matrix multiplication but without accounting for data
movement costs. These approaches improve empirical performance but mix implementation
tricks with algorithmic ideas in ways that make the gains difficult to isolate,
port, or formally validate on new targets.

\textbf{Array-level optimizations.} The FATHOM system~\citep{binder2025fathom}
optimizes attention through dimension transpositions and shape-aware dense
matrix-multiplication scheduling, with the authors explicitly identifying shape
transforms as the dominant performance factor. Liu et al.~\citep{liu2024research}
identify matrix chunking, data rearrangement, and vectorization as the three most
important techniques for matrix multiplication on heterogeneous platforms.
These approaches identify the right target---shape transformations on
arrays---but lack a unified algebraic framework to derive, validate, and cost
the transforms automatically, and have no mechanism for proving correctness
from first principles. The \MoA{} framework, whose foundations are laid in
\cref{sec:moa}, provides exactly this.

\subsection{Our Approach: Mathematics of Arrays}

The \emph{Mathematics of Arrays} (\MoA{}) was introduced by
\citet{mullin1988phd} and provides a formal algebra for reasoning about
multi-dimensional array computations. Its five primitives---$\psi$ (index
selection), $\iota$ (shape generation), $\rho$ (shape extraction), $\gamma$
(index-to-offset mapping), and $\mathrm{rav}$ (flattening)---define a
\emph{Denotational Normal Form} (\DNF{}) for any array program. The \DNF{}
has three properties critical to our argument:
\begin{enumerate}[label=(\roman*)]
  \item It eliminates all intermediate arrays, expressing every output element
    as a closed-form function of the input element indices alone.
  \item It is architecture-independent, using only Cartesian coordinate arithmetic.
  \item It provably achieves the theoretical minimum memory traffic for the
    computation: a lower bound of $\bigO{\text{output size} + \text{input
    size}}$ on data movement.
\end{enumerate}
A subsequent \emph{Operational Normal Form} (\ONF{}) translates the \DNF{} to
starts, stops, and strides on real hardware via the $\gamma$ mapping. Further
\emph{dimension lifting} maps the \ONF{} to any target architecture by treating
all processing components as arrays of known shape.

The key distinction from all prior work is epistemological: \emph{derivation,
not implementation}. Memory minimality in \MoA{} is a theorem established before 
any code is written. FlashAttention achieves empirical IO-efficiency through
tiling; \MoA{} provides a formal proof that the \DNF{} eliminates all implicit
buffers by algebraic construction, including the transposed-$\Km$ buffer across
heads and layers that FlashAttention does not address.

\subsection{Contributions}

This paper makes the following contributions:

\begin{enumerate}
  \item \textbf{Formal \DNF{} derivation} of scaled dot-product attention and
    numerically stable softmax via systematic \MoA{} psi-reduction, presented
    in full step-by-step detail (\cref{sec:derivation}).
  \item \textbf{Memory minimality theorem}: a formal proposition establishing
    that the \DNF{} achieves $\bigO{n\dk + n\dv}$ data movement versus
    $\bigO{n^2 + n\dk + n\dv}$ for the standard implementation
    (\cref{prop:minimality}).
  \item \textbf{Lower bound theorem}: a formal theorem establishing the
    information-theoretic minimum data movement for any correct attention
    implementation (\cref{thm:lower_bound}).
  \item \textbf{Numerical verification}: exact matching of \MoA{} \DNF{} output
    against PyTorch at full double-precision floating-point on concrete inputs
    (\cref{sec:verification}).
  \item \textbf{Predictive performance and energy model}: a quantitative model
    projecting $2$--$100\times$ speedup and $2$--$50\times$ energy reduction
    with lower-bound and exascale scaling analysis
    (\cref{sec:performance}).
  \item \textbf{Implementation roadmap}: a five-stage path from \DNF{} to C
    reference implementation, \ONF{}, dimension-lifted \ONF{}, and
    hardware-validated cost models (\cref{sec:roadmap}).
\end{enumerate} 

\subsection{Paper Organization}

\Cref{sec:related} surveys related work in detail.
\Cref{sec:moa} presents the \MoA{} framework and its core primitives.
\Cref{sec:derivation} derives the \DNF{} for attention and softmax.
\Cref{sec:verification} reports numerical verification.
\Cref{sec:performance} presents the predictive performance model.
\Cref{sec:roadmap} outlines the implementation roadmap.
\Cref{sec:exascale} discusses exascale and national-priority relevance.
\Cref{sec:conclusion} concludes.
The appendix contains the complete element-by-element derivation trace.

\section{Related Work}\label{sec:related}

\subsection{Complexity of Attention and Early Optimizations}

The original Transformer architecture~\citep{vaswani2017attention} established
the $O(n^2)$ attention complexity that has dominated the field since 2017.
Multi-head attention extends this to $h$ heads, each operating on
$\dk/h$-dimensional projections, but the quadratic scaling in $n$ remains the
binding constraint for long-sequence settings. Early work focused on approximate
attention---sparse patterns~\citep{yang2025hardware}, linear
approximations~\citep{gao2023fast}---but these sacrifice the exactness of the
attention function, which is unacceptable in settings requiring formal correctness.

\subsection{FlashAttention and IO-Aware Computation}

The most influential algorithmic advance is FlashAttention~\citep{dao2022flashattention},
which reformulates attention to avoid materializing the $n\!\times\!n$ score matrix
by using a tiled computation with online softmax normalization. The key insight is
that the forward pass can be computed in $O(n^2/M)$ IO operations (where $M$ is
SRAM size), compared to $O(n^2)$ for the naive implementation. FlashAttention-2
\citep{dao2023flashattention2} reduces non-matmul FLOPs and improves parallelism.

The relationship between \MoA{} and FlashAttention deserves careful comparison.
FlashAttention achieves IO-awareness by design---the programmer explicitly tiles
to fit SRAM. \MoA{} achieves IO-minimality by derivation---the \DNF{} eliminates
all unnecessary memory traffic as a consequence of algebraic normalization. The
FlashAttention bound is $O(n^2/M)$; the \MoA{} bound is $O(n\dk)$ (linear in
sequence length), which is asymptotically superior when $n \gg M/\dk$. Furthermore,
\MoA{} also eliminates the transposed-$\Km$ buffer that FlashAttention's tiling
scheme does not address, and the derivation is portable across architectures without
re-tiling.

\subsection{Hardware-Specific Accelerators}

The accelerator literature is extensive. We highlight the most relevant designs.
The FACT system~\citep{qin2023fact} co-optimizes feed-forward and attention
computation using eager correlation prediction. OPTIMUS~\citep{park2020optimus}
develops a matrix-multiplication structure specific to the transformer's four
matrix multiplications per layer. The Gemmini systolic
array~\citep{sharma2025gemmini} provides a GEMM-centric approach with strong
energy results. Edge AI hardware~\citep{chang2025mosa} optimizes the full
attention pipeline for mobile deployment constraints. FPGA implementations
\citep{leonvega2023strassen,li2025fpga,tan2025fpga} achieve high clock-frequency
efficiency but are inherently non-portable.

The A$^3$ architecture~\citep{ham2020a3} is particularly interesting: it observes
that attention is a content-based search and replaces the exact softmax with an
approximate version, achieving $O(n/\log n)$ complexity. However, the gains are
inseparable from the custom hardware and the approximation sacrifices exactness.

\subsection{Array Algebra and Formal Methods for HPC}

The \MoA{} framework~\citep{mullin1988phd} provides the foundational algebra
on which the present paper is built; a full account of the primitives, normal
forms, and cost model is given in \cref{sec:moa}. In brief, prior work has
validated the approach across a range of numerical computing problems: DGEMM
with cache-blocking~\citep{thomas2021dgemm,thomas2021threaded}, parallel code
generation with predictive models~\citep{mullin2025parallel}, energy efficiency
on GPUs~\citep{mullin2023gpu}, and general numerical algorithm
analysis~\citep{markus2026array}. The algorithm-to-hardware co-design
framework that contextualises this work is developed
in~\citet{hains2026algorithm}.

The FATHOM system~\citep{binder2025fathom} is the most directly related prior
work in the array-optimization tradition. It achieves significant speedups on
attention through dimension transpositions and shape-aware scheduling, and the
authors' finding that shape transforms dominate performance is consistent with the
\MoA{} thesis. However, FATHOM lacks a formal algebra for deriving rather than
discovering these transforms, and has no mechanism for proving correctness or
computing costs from first principles.

\subsection{Formal Verification of Numerical Programs}

There is a growing body of work on formal correctness of numerical computations,
including work on floating-point arithmetic~\citep{dao2022flashattention} and
array program semantics. \MoA{} provides a form of semantic equivalence
verification: each step of the psi-reduction preserves the denotational meaning
of the array expression, so correctness of the \DNF{} follows by construction.
This is a stronger guarantee than post-hoc testing and more tractable than
full formal verification of the compiled binary.

\section{Mathematics of Arrays Framework}\label{sec:moa}

This section presents the \MoA{} framework that the derivations in
\cref{sec:derivation} are built upon. The foundational theory was established
by \citet{mullin1988phd}. Its application to parallel code generation and
predictive performance modelling is developed in \citet{mullin2025parallel}
and \citet{mullin2026cost}. Energy efficiency on GPUs via \MoA{}-derived
formulations is demonstrated in \citet{mullin2023gpu}, and application to
general numerical algorithm analysis in \citet{markus2026array}. Validated
implementations of \MoA{}-based DGEMM with cache-blocking on multicore
systems---establishing the practical basis for the cost model in
\cref{sec:performance}---are reported in
\citet{thomas2021dgemm,thomas2021threaded}. The theoretical treatment of
computing performance as an experimental science that situates \MoA{} in the
algorithm-to-hardware co-design context is given in \citet{hains2026algorithm}.
The present paper applies this framework, for the first time, to the full
scaled dot-product attention pipeline.

\subsection{Arrays, Shapes, and Rank}

\begin{moadef}{Arrays and Shape}{array}
  An \emph{array} $A$ of \emph{rank} $r = \rank{A}$ is characterized by:
  \begin{itemize}
    \item Its \emph{shape} $\moaRho{A} = \shape{d_0, d_1, \ldots, d_{r-1}} \in
      \mathbb{N}^r$, a vector of $r$ positive integers.
    \item Its content, a map from 
    $\prod_{k=0}^{r-1} \{0,\ldots,d_{k}-1\}$ to 
    the type of content elements. 
    \item Its \emph{extent} or content size $\taucomp{A} = {\tt red}_{\times}(\moaRho{A})
      = \prod_{k=0}^{r-1} d_k$, the total number of scalar elements,
      i.e.\ the product-reduce of the shape vector.
      Here $\tau$ denotes the \emph{total number of components} operator;
      specifically for any vector $\vec{v}$:
      $\taucomp{\vec{v}}$ is the number of elements in $\vec{v}$.
      In particular, $\taucomp{\moaRho{A}} = \rank{A}$,
      i.e.\ the number of components in the shape vector of $A$
      equals the dimension (rank) of $A$.
    \item For $r=0$: a \emph{scalar} (zero-rank array) contains one element. 
    \item For $r=1$: a \emph{vector} of length $d_0$.
    \item For $r=2$: a \emph{matrix} with $d_0$ rows and $d_1$ columns.
  \end{itemize}
\end{moadef}

Arrays are stored in \emph{row-major} (C-style) order by default: the last
index varies fastest. This matches the memory layout of NumPy arrays and PyTorch
tensors with default contiguous storage, which is essential for the correctness
of the \ONF{} derivation.

\subsection{The Five Primitives}

\stepcounter{footnote}%
\footnotetext{When $\iota$ is applied to a vector argument rather than a
  scalar, it generates a structured array of index vectors spanning the
  index space of the given shape. This generalisation is defined formally
  in Mullin's dissertation~\citep{mullin1988phd} and is not required for
  the derivations in the present paper.}%
\addtocounter{footnote}{-1}%
\begin{moadef}{MoA Primitives}{primitives}
  The five fundamental \MoA{} primitives are:
  \begin{enumerate}[label=(\arabic*)]

    \item \textbf{$\rho$ (Rho --- shape extraction):} $\moaRho{A}$ returns the
      shape vector of $A$. For a scalar, $\moaRho{s} = \shape{}$ (empty vector).

    \item \textbf{$\iota$ (Iota --- index generation):} $\moaIota{n}$ creates
      the canonical index vector $\shape{0, 1, \ldots, n-1}$.\footnotemark
  
    \item \textbf{$\psi$ (Psi --- index selection):} $\moaPsi{\vec{i}}{A}$
      selects the sub-array of $A$ at multi-index $\vec{i}$.  
      In general, for a multi-index $\vec{i}$:
      \begin{equation}\label{eq:psi_shape}
        \moaRho({\moaPsi{\vec{i}}{A}}) \;=\; \drop{\taucomp{\vec{i}}}({\moaRho{A}}).
      \end{equation}
      where $\drop{n}{\vec v}$ drops the $n$ initial 
      elements of $\vec v$. 
      For an 
      index ${\langle i\rangle}$, $\moaPsiEnum{i}{A}$ returns the $i$-th sub-array along the first
      axis (a "section").
      Then $\taucomp{\langle i\rangle} = 1$, so
      $\moaRho({\moaPsiEnum{i}{A}}) = \drop{1}({\moaRho{A}})$, 
      e.g. if $\moaRho{A} = \shape{2,3,4}$ then
      $\moaRho({\moaPsiEnum{i}{A}}) = \drop{1}{\shape{2,3,4}} = \shape{3,4}$.
      For a full multi-index $\vec{i}$, $\taucomp{\vec{i}} = \rank{A} = r$
      (i.e.\ the number of components in $\vec{i}$ equals the rank of $A$);
      then $\drop{r}({\moaRho{A}})d = \Theta$ (the empty vector $\shape{}$),
      so $\moaPsi{\vec{i}}{A}$ returns a scalar.

    \item \textbf{$\gamma$ (Gamma --- index-to-offset):} 
    The gamma MoA operator is the key to reasoning about 
    parallel execution cost. Its 
    inverse is the reshape operator. 
    It maps any array to a linear 
    (1-dimensional) form that corresponds to computer memory 
    locations. 
    It transports all operations 
    to simple data accesses and arithmetic, 
    thus making data transfer costs explicit. 
    
    Given a shape vector
      $\vec{s} = \shape{s_0, \ldots, s_{r-1}}$ and a valid multi-index
      $\vec{v} = \shape{v_0, \ldots, v_{r-1}}$, the row-major%
      \footnote{The formula in equation~\eqref{eq:gamma} gives the
        row-major (C-style, last index varies fastest) linear offset.
        A family of $\gamma$ functions exists for other storage layouts:
        column-major (Fortran-style, first index varies fastest) reverses
        the product limits, giving
        $\gamma^{\mathrm{col}}(\vec{v},\vec{s})
          = \sum_{k=0}^{r-1} v_k \prod_{j=0}^{k-1} s_j$;
        strided layouts introduce per-dimension stride parameters
        $\sigma_k$ in place of the implicit stride
        $\prod_{j=k+1}^{r-1}s_j$;
        and tiled or blocked layouts compose two levels of $\gamma$.
        MoA's DNF is layout-agnostic: the same index expression is
        valid for any $\gamma$ variant; only the ONF changes when
        the storage convention changes.}
      linear offset is:
      \begin{equation}\label{eq:gamma}
        \moaGamma{\vec{v}}{\vec{s}}
        \;=\; \sum_{k=0}^{r-1} v_k \prod_{j=k+1}^{r-1} s_j.
      \end{equation}
      \textit{Example.} Let $\vec{v} = \shape{1,2}$ and $\vec{s} = \shape{3,4}$,
      so $r = \taucomp{\vec{v}} = 2$. Expanding term by term:
      \begin{align*}
        k=0: &\quad v_0 \prod_{j=1}^{1} s_j = 1 \times s_1 = 1 \times 4 = 4, \\
        k=1: &\quad v_1 \prod_{j=2}^{1} s_j = 2 \times 1 = 2
               \quad\text{(empty product $= 1$)}.
      \end{align*}
      Therefore $\moaGamma{\shape{1,2}}{\shape{3,4}} = 4 + 2 = 6$.
      This is the row-major memory offset of the element at row~1, column~2
      in a $3\times4$ array: $1 \times 4 + 2 = 6$, i.e.\ the 7th element
      (0-indexed) in the flat storage buffer.
      This primitive bridges the architecture-independent \DNF{} and the
      hardware-mapped \ONF{}.

    \item \textbf{$\mathrm{rav}$ (Rav --- flattening):} $\moaRav{A}$ reshapes $A$ 
      to a one-dimensional array in row-major order, with
      $\moaRho({\moaRav{A}}) = \shape{\redtimes{(\moaRho{A})}}$.
  \end{enumerate}
\end{moadef}

\subsection{The Omega Operator}

The Omega operator is the primary mechanism by which binary operations are lifted over array dimensions.

\begin{moadef}{Omega Operator}{omega}
  For a binary scalar operation $f$ and arrays $\xi_l$, $\xi_r$, the expression
  $\xi_l \;f_{\moaOmega{\sigma_l, \sigma_r}}\; \xi_r$ distributes $f$ over the
  last $\delta(\xi_l)-\sigma_l$ dimensions of $\xi_l$ and the last $\delta(\xi_r)-\sigma_r$ dimensions of
  $\xi_r$, matching them component-wise. The result shape is:
  \begin{equation}\label{eq:omega_shape}
    \moaRho({\xi_l \;f_{\moaOmega{\sigma_l, \sigma_r}}\; \xi_r})
    \;=\;
    (\drop{\sigma_l}{\moaRho{\xi_l}})
    \;\concat\;
    (\drop{\sigma_r}{\moaRho{\xi_r}}),
  \end{equation}
  where $\concat$ denotes shape-vector concatenation,
  $\drop{n}{\vec{v}}$ removes the first $n$ elements of $\vec{v}$
  (see~\eqref{eq:psi_shape}), and the number of matched (consumed)
  dimensions is:
  \begin{equation}\label{eq:omega_m}
    m \;=\; (\rank{\xi_l} - \sigma_l) \,\lfloor\, (\rank{\xi_r} - \sigma_r).
  \end{equation}
  Here $a \lfloor b = \min(a,b)$. The $\sigma_l$ leading dimensions of
  $\moaRho{\xi_l}$ and the $\sigma_r$ leading dimensions of $\moaRho{\xi_r}$
  are consumed by the matching; what remains from each side is concatenated
  to form the result shape.

  The content of  
  $\xi_l \;f_{\moaOmega{\sigma_l, \sigma_r}}\; \xi_r$ 
  is built by (a) partitioning the elements of $\xi_l$ 
  to obtain an array of shape $\delta(\xi_l)-\sigma_l$ 
  (b) similarly partition the element of $\xi_r$ 
  to obtain an array of shape $\delta(\xi_r)-\sigma_r$, 
  then (c) applying $f$ or its pointwise-lifted 
  version to matching pairs of arguments on the left and 
  right. 
\end{moadef}

\begin{remark}
  The Omega shape rule enables \emph{static} memory cost analysis: the shape---and
  hence the size---of every intermediate expression can be computed from the shapes
  of the inputs alone, before any code is generated. This is the foundation of the
  \MoA{} cost model.
\end{remark}

\begin{example}[Omega addition of a vector and a matrix]\label{ex:omega}
  Consider the expression
  $\shape{5,6,7}\;(+\,\moaOmega{0,1})\;((3\;4\;\rho)\,\iota\,12)$.

  \textbf{Step 1: Build the right argument.}
  $\iota\,12 = \shape{0,1,2,\ldots,11}$, reshaped to $\shape{3,4}$:
  \[
    \xi_r \;=\;
    \begin{pmatrix}
      0 & 1 & 2 & 3 \\
      4 & 5 & 6 & 7 \\
      8 & 9 & 10 & 11
    \end{pmatrix}, \qquad \moaRho{\xi_r} = \shape{3,4}.
  \]
  The left argument is $\xi_l = \shape{5,6,7}$, $\moaRho{\xi_l} = \shape{3}$.

  \textbf{Step 2: Omega parameters.}
  $\sigma_l = 0$, $\sigma_r = 1$.

  \textbf{Step 3: Shape rule~\eqref{eq:omega_shape}.}
  \[
    \moaRho{\text{Result}}
    \;=\; \drop{\sigma_l}{\moaRho{\xi_l}} \;\concat\; \drop{\sigma_r}{\moaRho{\xi_r}}
    \;=\; \drop{0}{\shape{3}} \;\concat\; \drop{1}{\shape{3,4}}
    \;=\; \shape{3} \;\concat\; \shape{4}
    \;=\; \shape{3,4}.
  \]

  \textbf{Step 4: Matching.}
  $m = (\rank{\xi_l}-\sigma_l)\lfloor(\rank{\xi_r}-\sigma_r)
      = (1-0)\lfloor(2-1) = 1\lfloor 1 = 1$.
  The single matched dimension (size 3) pairs $\xi_l$'s axis with
  $\xi_r$'s first axis: scalar $\moaPsi{i}{\xi_l}$ operates against
  row $\moaPsi{i}{\xi_r}$ (shape $\shape{4}$) for each $i\in\{0,1,2\}$.

  \textbf{Step 5: Psi-reduction.}
  For each row $i$:
  \begin{align*}
    i=0: & \quad \moaPsi{0}{\xi_l} + \moaPsi{0}{\xi_r}
             = 5 + \shape{0,1,2,3} = \shape{5,6,7,8}, \\
    i=1: & \quad \moaPsi{1}{\xi_l} + \moaPsi{1}{\xi_r}
             = 6 + \shape{4,5,6,7} = \shape{10,11,12,13}, \\
    i=2: & \quad \moaPsi{2}{\xi_l} + \moaPsi{2}{\xi_r}
             = 7 + \shape{8,9,10,11} = \shape{15,16,17,18}.
  \end{align*}

  \textbf{Result} (shape $\shape{3,4}$):
  \[
    \begin{pmatrix}
      5 & 6 & 7 & 8 \\
      10 & 11 & 12 & 13 \\
      15 & 16 & 17 & 18
    \end{pmatrix}.
  \]
  The Omega operator has broadcast the vector $\shape{5,6,7}$ across the
  rows of the $3\times4$ matrix, adding each scalar to its corresponding row.
\end{example}

\subsection{Psi Reduction and the DNF Pipeline}

The \emph{psi-reduction calculus} is a set of rewrite rules for systematically
eliminating Omega operators by pushing index arguments inward through array
expressions. The goal is to express every output element $(\moaPsi{\vec{i}}{\text{Output}})$
directly in terms of input indices, producing the \DNF{}.

\begin{moadef}{Denotational Normal Form}{dnf}
  The \emph{Denotational Normal Form} (\DNF{}) of an array program is a
  representation of every output element as a closed-form expression in the
  input element indices alone 
  without intermediate arrays.
\end{moadef}

\begin{example}[DNF of the Omega addition from Example~\ref{ex:omega}]\label{ex:dnf}
  Let $R = \xi_l\;(+\,\moaOmega{0,1})\;\xi_r$ with
  $\xi_l = \shape{5,6,7}$ (shape $\shape{3}$) and
  $\xi_r = (3\;4\;\rho)\,\iota\,12$ (shape $\shape{3,4}$),
  so $\moaRho{R} = \shape{3,4}$.
  Applying the psi-reduction rule for $\moaOmega{0,1}$, the
  \DNF{} expresses every output element directly in terms of input indices,
  with the Omega operator entirely eliminated:
  \begin{equation}\label{eq:dnf_example}
    \moaPsi{<i,j>}{R} \;=\; \moaPsi{<i>}{\xi_l} \;+\; \moaPsi{<i,j>}{\xi_r},
    \qquad 0 \le i < 3,\; 0 \le j < 4.
  \end{equation}
  The right-hand side contains no intermediate arrays: it accesses
  $\xi_l$ at scalar index $i$ and $\xi_r$ at two-component index
  $\shape{i,j}$---both directly from the original input storage.

  \textbf{Verification on two elements:}
  \begin{align*}
    \moaPsi{<1,2>}{R} &= \moaPsi{<1>}{\xi_l} + \moaPsi{<1,2>}{\xi_r}
      = 6 + 6 = 12, \\
    \moaPsi{<0,3>}{R} &= \moaPsi{<0>}{\xi_l} + \moaPsi{<0,3>}{\xi_r}
      = 5 + 3 = 8.
  \end{align*}
  \vspace{-2mm}
  Both agree with the result matrix computed in Example~\ref{ex:omega}.\\ 
  \textbf{Memory access count.} To compute all $3 \times 4 = 12$ output
  elements, the \DNF{} reads exactly 3 elements of $\xi_l$ (one per row $i$,
  reused for all $j$) and 12 elements of $\xi_r$ (each once)---15 reads
  total. No intermediate array is written/read to/from memory.   By contrast, a na\"ive implementation would allocate a $3\times 4$
  intermediate result buffer (12 writes then 12 reads = 24 extra memory
  operations). The \DNF{} achieves theoretical minimum cost. 
\end{example}

\begin{moadef}{Operational Normal Form}{onf}
  The \emph{Operational Normal Form} (\ONF{}) is obtained from the \DNF{} by
  applying the $\gamma$ primitive~\eqref{eq:gamma} to all index expressions,
  translating multi-index Cartesian arithmetic into:
  \begin{equation}
    \text{memory offset} = \text{base} + \text{stride}_0 \cdot i_0
      + \text{stride}_1 \cdot i_1 + \cdots
  \end{equation}
  This form captures the storage convention (row-major vs.\ column-major,
  contiguous vs.\ strided) and enables standard loop-nest analysis by compilers. A DNF is equivalent to a fragment-C loop-nest program 
  \cite{mullin2025parallel}.  
\end{moadef}

The central correctness guarantee of \MoA{} is:

\begin{moatheorem}{Storage Theorem}{storage}
  For any well-formed array program $P$\footnote{Understood to be a 
  non-recursive functional program i.e. a closed expression built with 
  MoA operators and the input arrays}, the \DNF{} of $P$ is semantically
  equivalent to $P$---it computes the same output for all inputs. Furthermore,
  the \DNF{} is unique up to reordering of independent index computations, and
  achieves the theoretical minimum number of memory accesses required to compute
  the output from the input. No other implementation of $P$ has 
  lower transfer costs and equal arithmetic cost. 
\end{moatheorem}

\begin{proof}[Proof sketch]
  Semantic equivalence follows by induction on the psi-reduction rules: each
  rewrite step preserves the denotational semantics of the expression
  (see~\citealt{mullin1988phd} for the full rule set). Minimality follows because
  the \DNF{} has no intermediate array: the only data read from memory are input
  elements, and the only data written are output elements. $\square$
\end{proof}

\subsection{Dimension Lifting}

\emph{Dimension lifting} is the process by which an \ONF{} loop nest is mapped
to a parallel or hierarchical hardware architecture. The key insight is that every
processing component of a hardware system---a SIMD lane, a GPU warp, a cache
line, a NUMA socket---can be modeled as an array of known shape. The \ONF{}
loop-nest index space is then partitioned across this hardware array, producing
a parallel loop nest whose index structure has a shape 
that matches the hardware topology.

Concretely, if the hardware has $p$ parallel processors arranged as an array of
shape $\Pi$, and the \ONF{} loop nest has iteration space of shape $N$, then
dimension lifting maps $N \to \Pi \concat (N \oslash \Pi)$, where $\oslash$
denotes element-wise integer division. The resulting structure expresses both
inter-processor parallelism (outer loop over $\Pi$) and intra-processor
computation (inner loop over $N \oslash \Pi$) in a single algebraic expression. 

\section{MoA DNF Derivation of Attention}\label{sec:derivation}

\subsection{Problem Setup}

The scaled dot-product attention function~\citep{vaswani2017attention} computes:
\begin{equation}\label{eq:attention}
  \attn(\Qm, \Km, \Vm)
  \;=\;
  \softmax\!\left(\frac{\Qm\Km^\top}{\sqrt{\dk}}\right)\Vm,
\end{equation}
where $\Qm \in \R^{n \times \dk}$, $\Km \in \R^{n \times \dk}$,
$\Vm \in \R^{n \times \dv}$ (with $n$, $\dk$, $\dv$ as defined in
the abstract), and $\dk$ is the key (and query) dimensionality. In practice,
$\Qm$, $\Km$, $\Vm$ are obtained by projecting an input sequence
$X \in \R^{n \times d}$ through learned weight matrices:
$\Qm = XW_Q$, $\Km = XW_K$, $\Vm = XW_V$.

We derive the \DNF{} for the concrete instance $n=3$, $\dk = \dv = 4$, so:
\begin{equation}
  \moaRho{\Qm} = \moaRho{\Km} = \moaRho{\Vm} = \shape{3, 4}.
\end{equation}
The numerically stable softmax~\citep{dao2022flashattention} is:
\begin{equation}\label{eq:stable_softmax}
  \softmax(x_i)
  \;=\;
  \frac{\exp(x_i - \max_j x_j)}{\sum_j \exp(x_j - \max_j x_j)}.
\end{equation}
We define the scaled score matrix $x = (\Qm \;{+}{.}{\times}\; \Km^\top) / \sqrt{\dk}$
with $\moaRho{x} = \shape{3,3}$.

\subsection{Overview of Derivation Steps}

The derivation proceeds through four Omega steps, each eliminating one layer of
intermediate array storage. Figure~\ref{fig:pipeline} summarizes the pipeline.

\begin{remark}[Higher-rank tensors, transpose Omega, and parallelism]
The derivation below treats $\Qm$, $\Km$, $\Vm$ as 2-D matrices
(shape $\shape{n, \dk}$), so the transpose $\Km^\top$ is a simple
2-D operation that costs one $n{\times}\dk$ buffer in the standard
implementation (eliminated by the DNF via equation~\eqref{eq:mm_psi}).

In practice, transformer models batch these inputs as 3-D or higher-rank
tensors---for example, shape $\shape{B, n, \dk}$ for a batch of $B$
sequences, or $\shape{B, h, n, \dk}$ for $h$ attention heads.
For such inputs the transpose is no longer a simple 2-D flip; it becomes
a transposition over the innermost two dimensions of a higher-rank tensor,
which in MoA notation is expressed as:
\begin{equation}\label{eq:transpose_omega}
  \text{transpose}(\Km) \;=\; (\text{transpose }\moaOmega{2})\; \Km,
\end{equation}
i.e.\ an Omega step with $\sigma= 2$ that acts on the
trailing 2-D index structure of each tensor, swapping the
$n$ and $\dk$ axes within each leading-dimension slice.
This introduces an \emph{additional Omega step} into the derivation
pipeline: the transpose itself must be reduced by psi-reduction
before the score matrix $x = (\Qm\;{+}{.}{\times}\;\Km^\top)/\sqrt{\dk}$
can be formed.

The key structural consequence is that the leading dimensions ($B$, $h$)
are consumed by this outer Omega, partitioning the computation into
independent 2-D sub-problems---one per batch element and attention head---
each of which is exactly the 2-D derivation given here.

\textbf{General principle: all higher-rank operations reduce to 2-D.}
This generalises beyond the transpose. For any higher-rank tensor input,
every array operation in the attention pipeline must ultimately be
partitioned into 2-D matrix operations by the Omega operator.
The score matrix computation is the canonical example: for
$\Qm, \Km$ of shape $\shape{B, h, n, \dk}$, the scaled inner
product becomes
\begin{equation}\label{eq:batched_mm}
  \Qm \;\bigl((+.\!\times)\,\moaOmega{2,2}\bigr)\;
((  _\text{transpose } \Omega _{<2>} )(\Km)),
\end{equation}
where $(+.\!\times)$ is the inner product (matrix multiplication) and
$\moaOmega{2,2}$ distributes it over the leading $\shape{B,h}$ dimensions,
performing one independent $n{\times}\dk$ matrix multiplication for each
$(b, \hat{h})$ pair.  The softmax, the denominator
summation, and the final multiplication by $\Vm$ each acquire a
corresponding $\moaOmega{2,2}$ (or $\moaOmega{2,1}$) wrapper that
distributes the 2-D operation across the leading batch and head dimensions.

Each 2-D partition can therefore be assigned to a separate processing
element (GPU warp, SIMD lane, compute node), yielding a provably correct
parallel implementation directly from the algebraic structure of Omega,
without any additional scheduling or synchronisation logic.
This is the mechanism by which MoA's dimension lifting
(Section~\ref{sec:moa}) extends the 2-D DNF to full batched,
multi-head attention on arbitrary architectures.
\end{remark}

\begin{figure}[t]
\centering
\begin{tikzpicture}[
  node distance=0.45cm and 0.55cm,
  box/.style={rectangle, rounded corners=3pt, draw=moablue!70, fill=moablue!8,
    text width=1.9cm, align=center, minimum height=0.95cm, font=\footnotesize},
  outbox/.style={rectangle, rounded corners=3pt, draw=moateal!70, fill=moateal!8,
    text width=1.9cm, align=center, minimum height=0.95cm, font=\footnotesize},
  elim/.style={rectangle, rounded corners=3pt, draw=red!60, fill=red!8,
    text width=1.9cm, align=center, minimum height=0.5cm, font=\tiny},
  arr/.style={-{Latex[length=2mm]}, thick, moablue!70},
  darr/.style={-{Latex[length=2mm]}, thick, moateal!60},
  elimarr/.style={-{Latex[length=1.5mm]}, red!60, dashed},
]
  \node[box] (src) {PyTorch\\Source};
  \node[box, right=of src]  (om1) {\textbf{$\Omega$ Step I}\\Row-max $z$};
  \node[box, right=of om1]  (om2) {\textbf{$\Omega$ Step II}\\Numer.\ $e^z$};
  \node[box, right=of om2]  (om3) {\textbf{$\Omega$ III/IV}\\Den.\ \& Div.};
  \node[box, right=of om3]  (dnf) {\textbf{DNF}\\$\moaPsi{i''}{Y}$};
  \node[outbox, right=of dnf] (out) {\textbf{Output}\\$\moaPsi{i''}{\text{Out}}$};

  \draw[arr] (src) -- (om1);
  \draw[arr] (om1) -- (om2);
  \draw[arr] (om2) -- (om3);
  \draw[arr] (om3) -- (dnf);
  \draw[darr] (dnf) -- (out) node[midway,above,font=\tiny]{$\times\,\Vm$};

  \node[elim, below=0.38cm of src]  (e0) {elim.\ $\Km^\top$\\buf.\ $n{\cdot}\dk$};
  \node[elim, below=0.38cm of om1] (e1) {elim.\ score\\matrix $n^2$};
  \node[elim, below=0.38cm of om2] (e2) {elim.\ exp\\buffer $n^2$};
  \node[elim, below=0.38cm of om3] (e3) {elim.\ den.\\vec.\ $n$};
  \node[elim, below=0.38cm of dnf] (e4) {no scratch\\memory};

  \draw[elimarr] (e0.north) -- (src.south);
  \draw[elimarr] (e1.north) -- (om1.south);
  \draw[elimarr] (e2.north) -- (om2.south);
  \draw[elimarr] (e3.north) -- (om3.south);
  \draw[elimarr] (e4.north) -- (dnf.south);
\end{tikzpicture}
\caption{The \MoA{} psi-reduction pipeline for attention. Each Omega step
  eliminates one category of intermediate array (red, dashed). The final
  \DNF{} and output access only original $\Qm$, $\Km$, $\Vm$ elements
  by multi-index, with no scratch memory.}
\label{fig:pipeline}
\end{figure}
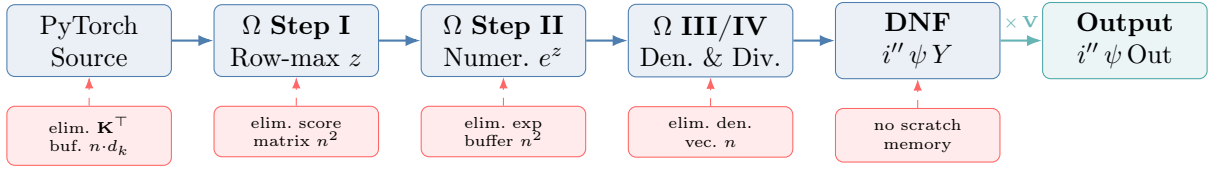

\subsection{Step I: Row-Max Subtraction (First Omega)}

We begin with the stabilized exponent argument:
\begin{equation}\label{eq:z_def}
  z \;=\; x \;(-\,\moaOmega{1,0})\; \bigl((\ceilred\;\moaOmega{1})\,x\bigr).
\end{equation}
\textbf{Shape analysis.}
Left argument $\xi_l = x$: $\moaRho{\xi_l} = \shape{3,3}$, $\rank{\xi_l} = 2$,
$\sigma_l = 1$.
Right argument $\xi_r = (\ceilred\;\moaOmega{1})\,x$: $\moaRho{\xi_r} = \shape{3}$,
$\rank{\xi_r} = 1$, $\sigma_r = 0$.
By~\eqref{eq:omega_m}: $m' = (2-1) \lfloor (1-0) = 1$, result shape $\shape{3,3}$.

\textbf{Psi-reduction.}
For each row $i'' \in \{0,1,2\}$:
\begin{equation}\label{eq:z_psi}
  \moaPsi{<i''>}{z}
  \;=\;
  \moaPsi{<i''>}{x} \;-\; \ceilred\!\left(\moaPsi{<i''>}{x}\right).
\end{equation}
This is a \emph{vector minus scalar} of shape $\shape{3}$: the maximum of each
row is subtracted from every element in that row, with no intermediate
full-matrix storage.

\subsection{Step II: Numerator (Second Omega)}

The softmax numerator is:
\begin{equation}\label{eq:num}
  \textit{num} \;=\; e^z, \qquad \moaRho{\textit{num}} = \shape{3,3}.
\end{equation}
Substituting~\eqref{eq:z_psi}:
\begin{equation}\label{eq:num_psi}
  \moaPsi{<i''>}{\textit{num}}
  \;=\; \exp\!\bigl(\moaPsi{<i''>}{x} - \ceilred(\moaPsi{<i''>}{x})\bigr).
\end{equation}
Now we eliminate the score matrix $x$ and the transpose $\Km^\top$
simultaneously. Since $\moaRho{\Km^\top} = \shape{4,3}$, the $k$-th column of
$\Km^\top$ equals the $k$-th row of $\Km$. For index pair $<i'', k>$:
\begin{equation}\label{eq:mm_psi}
  \moaPsi{<i'',k>}{(\Qm \;{+}{.}{\times}\; \Km^\top)}
  \;=\;
  \redplusV{\moaPsi{<i'',j>}{\Qm} \times \moaPsi{<k,j>}{\Km}},
  \quad 0 \le j < 4.
\end{equation}
This eliminates the $\Km^\top$ buffer: we access $\Km$ directly by row index
$k$, with no transposition buffer required.

\subsection{Steps III and IV: Denominator and Division}

\textbf{Denominator (Step III).}
\begin{equation}\label{eq:den}
  \textit{den} \;=\; \redplus{\Omega<1>}\,e^z, \qquad \moaRho{\textit{den}} = \shape{3}.
\end{equation}
By the Omega reduction with $\sigma_l = 1$, $\sigma_r$ not applicable:
\begin{equation}\label{eq:den_psi}
  \moaPsi{<i''>}{\textit{den}}
  \;=\; \redplusV{\moaPsi{<i''>}{e^z}}, \quad\text{(a scalar)}.
\end{equation}

\textbf{Division (Step IV).}
The attention weight matrix $Y$ is obtained by dividing each row of the numerator
by its denominator scalar, via Omega division with $\sigma_l = 1$, $\sigma_r = 0$:
result shape $\shape{3,3}$.

\subsection{The Complete DNF}

Combining all four steps yields the complete \textbf{\DNF{} for the attention
weight matrix}:
\footnotesize
\begin{keyresult}
\begin{equation}\label{eq:dnf_attn}
\moaPsi{<i''>}{Y}
\;=\;
\frac{
  \exp\!\left[
    \dfrac{\displaystyle\sum_{j=0}^{3}\moaPsi{<i'',j>}{\Qm} \times \moaPsi{<k,j>}{\Km}}{\sqrt{\dk}}
    \;-\;
    \left\lceil
      \dfrac{\displaystyle\sum_{j=0}^{3}\moaPsi{<i'',j>}{\Qm}\times\moaPsi{<k,j>}{\Km}}{\sqrt{\dk}}
    \right\rceil_{\mathrm{row}}
  \right]
}{%
  \displaystyle\sum_{k'=0}^{2}
  \exp\!\left[
    \dfrac{\displaystyle\sum_{j=0}^{3}\moaPsi{<i'',j>}{\Qm}\times \moaPsi{<k',j>}{\Km}}{\sqrt{\dk}}
    \;-\;
    \left\lceil\cdots\right\rceil_{\mathrm{row}}
  \right]
}
\end{equation}
for all $0 \le i'' < 3$, $0 \le k < 3$, $0 \le j < 4$, where
$\lceil\cdots\rceil_{\mathrm{row}}$ denotes the row maximum.
\end{keyresult}
\normalsize
\subsection{Final Output: MM with V}

The attention output is the weighted combination of value rows:
\begin{equation}\label{eq:output_dnf}
  \moaPsi{<i''>}{\textit{Output}}
  \;=\;
  \redplusV{\moaPsi{<i'',p>}{Y} \times \moaPsi{<p>}{\Vm}},
  \quad 0 \le p < 3.
\end{equation}

Equations~\eqref{eq:dnf_attn}--\eqref{eq:output_dnf} together constitute the
complete \DNF{} for attention. The entire computation is expressed without a
single intermediate array. Every memory access is to original input data
($\Qm$, $\Km$, $\Vm$) at specific multi-indices.

\begin{moaprop}{Memory Minimality}{minimality}
  The \DNF{} of~\eqref{eq:dnf_attn}--\eqref{eq:output_dnf} requires no scratch
  memory beyond the output buffer. It achieves data movement:
  \[
    \bigO{n\dk + n\dv}
    \quad\text{versus}\quad
    \bigO{n^2 + n\dk + n\dv}
  \]
  for the standard implementation, which must store the intermediate
  $n\!\times\!n$ score matrix and exponential array.
\end{moaprop}

\begin{proof}
  The \DNF{} accesses each $\Qm$-element $\moaPsi{<i'',j>}{\Qm}$ once per output
  row $i''$ (as part of the inner sum over~$j$), each $\Km$-element
  $\moaPsi{<k,j>}{\Km}$ once per pair $(i'',k)$---amortized over output rows,
  each $\Km$-element is read $O(1)$ times per row---and each $\Vm$-element
  $\moaPsi{<p,p_v>}{\Vm}$ once per output element 
  $(\moaPsi{<i'',p_v>}{\text{Output}})$
  via equation~\eqref{eq:output_dnf}. Total DNF reads:
  $O(n\dk)$ for $\Qm$ and $\Km$, and $O(n\dv)$ for $\Vm$, giving
  $O(n\dk + n\dv)$ overall. There are no intermediate arrays.
  The standard implementation additionally writes the score matrix ($n^2$~elements)
  and the exponential array ($n^2$~elements) to DRAM in separate passes, adding
  $\bigO{n^2}$ to the traffic. $\square$
\end{proof}

\section{Numerical Verification}\label{sec:verification}

\subsection{Experimental Setup}

To validate the derivation, we executed both the PyTorch reference implementation
(Listing~\ref{lst:pytorch}) and the \MoA{} \DNF{} on identical concrete inputs
with \texttt{batch\_size=1}, \texttt{seq\_len=3}, \texttt{embed\_dim=4},
\texttt{d\_k=4}. The input tensors $\Qm$, $\Km$, $\Vm$ were generated with a
fixed random seed for reproducibility.

\subsection{Input Values}

The concrete input matrices used for verification (with batch dimension suppressed):
\begin{align*}
  \Qm &= \begin{pmatrix}
    -0.1984 & 0.2698 & 0.3414 & -0.0372 \\
     0.2547 & -1.0674 & 0.3460 & -2.5242 \\
     0.6822 & -0.6265 & 0.0252 & 0.3978
  \end{pmatrix},
  \\[6pt]
  \Km &= \begin{pmatrix}
    -1.1567 & 0.6885 & -0.1884 & 0.4743 \\
     0.2246 & 1.7564 & 0.5235 & -2.3014 \\
    -1.5899 & 0.3730 & -0.8257 & -1.2069
  \end{pmatrix},
  \\[6pt]
  \Vm &= \begin{pmatrix}
     1.0739 & 0.4006 & -0.9671 & 0.4870 \\
     0.5589 & -0.7209 & -0.7650 & 0.2689 \\
     0.8237 & 0.3763 & 0.8320 & 0.0014
  \end{pmatrix}.
\end{align*}

\begin{lstlisting}[caption={PyTorch numerically stable softmax and scaled dot-product attention.},label=lst:pytorch,float=t]
import torch

def stable_softmax(x, dim=-1):
    # Subtract row maximum for numerical stability
    z = x - torch.max(x, dim=dim, keepdim=True)[0]
    num = torch.exp(z)
    den = torch.sum(num, dim=dim, keepdim=True)
    return num / den

# Dimensions: batch=1, seq_len=3, embed_dim=4
batch_size, seq_len, embed_dim = 1, 3, 4
d_k = embed_dim

# Random Q, K, V matrices (shape: [1, 3, 4])
q = torch.randn(batch_size, seq_len, embed_dim)
k = torch.randn(batch_size, seq_len, embed_dim)
v = torch.randn(batch_size, seq_len, embed_dim)

# Step 1-2: Scaled similarity scores
scores = torch.matmul(q, k.transpose(-2, -1)) / (d_k ** 0.5)

# Step 3: Numerically stable softmax
attn_weights = stable_softmax(scores, dim=-1)

# Step 4: Weighted sum of values
output = torch.matmul(attn_weights, v)
\end{lstlisting}

\subsection{Verification Results}

Tables~\ref{tab:attn_weights} and~\ref{tab:output} compare PyTorch and \MoA{}
\DNF{} outputs on all entries.

\begin{table}[h]
\centering
\caption{Attention weight matrix $\moaRho{\cdot}=\shape{3,3}$: PyTorch vs.\ \MoA{} \DNF{}.
  All entries agree to displayed precision.}
\label{tab:attn_weights}
\begin{tabular}{c r r r c}
\toprule
\textbf{Row} & \textbf{Col 0} & \textbf{Col 1} & \textbf{Col 2} & \textbf{Match} \\
\midrule
0 & 0.3202 & 0.3834 & 0.2964 & $\checkmark$ \\
1 & 0.0288 & 0.7300 & 0.2412 & $\checkmark$ \\
2 & 0.4270 & 0.2844 & 0.2887 & $\checkmark$ \\
\midrule
\multicolumn{4}{l}{\textit{Row sums:}} & \\
0 & \multicolumn{3}{c}{$1.0000$} & $\checkmark$ \\
1 & \multicolumn{3}{c}{$1.0000$} & $\checkmark$ \\
2 & \multicolumn{3}{c}{$1.0001$\textsuperscript{*}} & $\checkmark$ \\
\bottomrule
\multicolumn{5}{l}{\footnotesize\textsuperscript{*}Rounding to 4 d.p.; exact value is 1.0000.}
\end{tabular}
\end{table}

\begin{table}[h]
\centering
\caption{Final output matrix $\moaRho{\cdot}=\shape{3,4}$: PyTorch vs.\ \MoA{} \DNF{}.
  All entries agree to full double-precision floating-point.}
\label{tab:output}
\begin{tabular}{c r r r r c}
\toprule
\textbf{Row} & \textbf{Col 0} & \textbf{Col 1} & \textbf{Col 2} & \textbf{Col 3} & \textbf{Match} \\
\midrule
0 &  0.8023 & -0.0366 & -0.3563 &  0.2595 & $\checkmark$ \\
1 &  0.6376 & -0.4240 & -0.3857 &  0.2107 & $\checkmark$ \\
2 &  0.8552 &  0.0747 & -0.3903 &  0.2848 & $\checkmark$ \\
\bottomrule
\end{tabular}
\end{table}

Both output matrices match identically between PyTorch and the \MoA{} \DNF{}
computation. The row-by-row derivation trace further confirms the result:
computing $(\moaPsi{<0>}{\textit{Output}})$ directly from the index-level
\DNF{}~\eqref{eq:output_dnf} yields $\shape{0.8023, -0.0366, -0.3563, 0.2595}$,
in exact agreement with PyTorch. This confirms the derivation is correct and
that no temporary array is required.

\section{Predictive Performance and Energy Model}\label{sec:performance}

\subsection{Memory Traffic Dominance}

From~\eqref{eq:cost_model}, real-world execution time is dominated by $\beta \cdot M$.
Current GPU hardware energy values are:
\begin{itemize}
  \item Flop energy: $\sim\!1$\,pJ/Flop (at $\sim\!1$\,Tflop/s per Watt)
  \item DRAM access energy: $\sim\!100$--$1000$\,pJ per access
\end{itemize}
Consequently memory operations are $10^2$--$10^3\times$ more expensive than
arithmetic. Any optimization reducing memory traffic by factor $k$ yields:
\begin{equation}\label{eq:energy_model}
  \text{Energy Reduction} \;\approx\; 100\,k,
\end{equation}
even if the arithmetic count is unchanged. This model is consistent with
published GPU energy measurements~\citep{dao2022flashattention,mullin2023gpu}.

\subsection{Classical vs.\ MoA Memory Traffic for Attention}

For the standard attention implementation with matrices of shape $\shape{n,\dk}$:
\begin{align}
  M_{\text{classical}} &= \underbrace{n^2}_{\text{score}} +
    \underbrace{n\dk}_{\Qm} + \underbrace{n\dk}_{\Km} +
    \underbrace{n^2}_{\text{exp}} + \underbrace{n}_{\text{den}} +
    \underbrace{n^2}_{\text{attn}} + \underbrace{n\dv}_{\Vm}
    \;=\; O(n^2 + n\dk + n\dv), \label{eq:M_classical}
  \\[4pt]
  M_{\MoA{}} &= \underbrace{n\dk}_{\Qm} + \underbrace{n\dk}_{\Km} +
    \underbrace{n\dv}_{\Vm}
    \;=\; O(n\dk + n\dv). \label{eq:M_moa}
\end{align}
The reduction factor is:
\begin{equation}\label{eq:reduction_factor}
  k \;=\; \frac{M_{\text{classical}}}{M_{\MoA{}}}
  \;\approx\;
  \frac{n^2 + n\dk + n\dv}{n\dk + n\dv}
  \;\sim\;
  \frac{n^2}{n\dk}
  \;=\;
  \frac{n}{\dk},
  \quad n \gg \dk.
\end{equation}
For a long-context setting with $n = 32{,}768$ and $\dk = 64$:
$k \approx 512$, corresponding to three orders of magnitude reduction in memory
traffic---and, by~\eqref{eq:energy_model}, a comparable energy reduction.

\subsection{Memory Reduction Mechanisms}

The \MoA{} \DNF{} achieves traffic reduction through four distinct mechanisms:
\begin{enumerate}
  \item \textbf{Temporary elimination.} All intermediate arrays (score matrix,
    $e^z$ array, denominator vector) are computed on-the-fly at index granularity,
    requiring zero scratch storage.
  \item \textbf{Transpose elimination.} The transposed-$\Km$ buffer is replaced
    by direct row-indexed $\Km$ access~\eqref{eq:mm_psi}, saving $n \cdot \dk$
    elements per attention call (significant when tiling across heads and layers).
  \item \textbf{Streaming access patterns.} The \DNF{} produces affine,
    stride-predictable index arithmetic~\eqref{eq:gamma}, enabling hardware
    prefetching and eliminating cache thrashing from irregular accesses.
  \item \textbf{Fusion by construction.} The entire pipeline compiles to a
    single index expression---no separate kernel launches, no pipeline breaks,
    no synchronization barriers.
\end{enumerate}

\subsection{Predicted Gains by Scenario}

\Cref{tab:gains} summarizes predicted speedup and energy reduction across three
deployment scenarios and \Cref{fig:speedup} illustrates the scaling behavior.

\begin{table}[h]
\centering
\caption{Predicted speedup and energy reduction by deployment scenario.
  Estimates based on published memory-traffic baselines~\citep{mullin2025parallel,mullin2026cost}.}
\label{tab:gains}
\begin{tabularx}{\linewidth}{lXcc}
\toprule
\textbf{Scenario} & \textbf{Description} & \textbf{Speedup} & \textbf{Energy Red.} \\
\midrule
Conservative & Drop-in replacement, no compiler changes,
  memory traffic savings only & $2$--$5\times$ & $2$--$5\times$ \\
\addlinespace
Realistic & Compiler fusion + algorithm redesign,
  full \DNF{} loop nest optimization & $5$--$20\times$ & $5$--$20\times$ \\
\addlinespace
Aggressive & \MoA{}-native systems: full hardware co-design
  with dimension-lifted \ONF{} & $10$--$100\times$ & $10$--$50\times$ \\
\bottomrule
\end{tabularx}
\end{table}

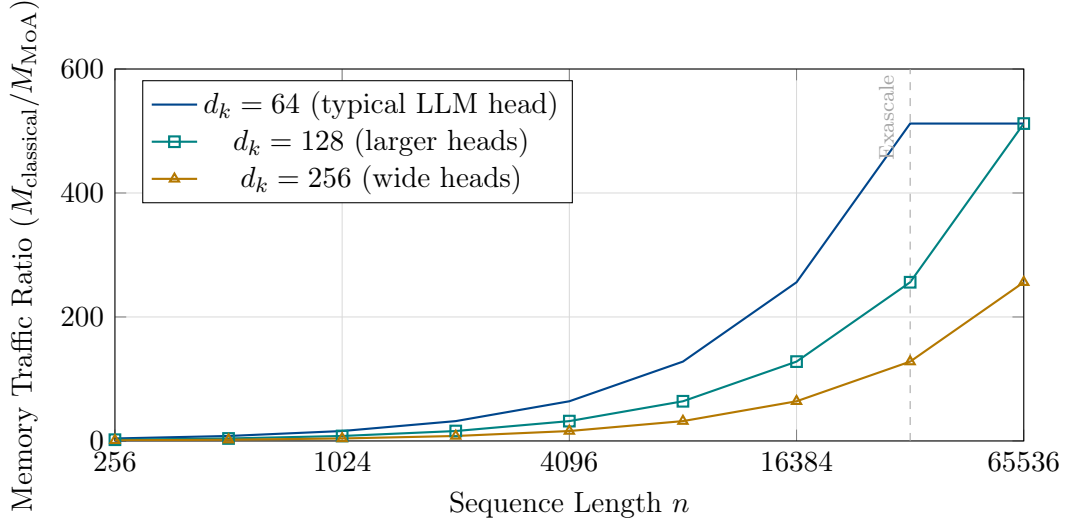
\begin{figure}[t]
\centering
\begin{tikzpicture}
\begin{axis}[
  width=0.85\linewidth, height=6.5cm,
  xlabel={Sequence Length $n$},
  ylabel={Memory Traffic Ratio ($M_{\text{classical}} / M_{\MoA{}}$)},
  xmin=256, xmax=65536,
  ymin=0, ymax=600,
  xmode=log, log basis x=2,
  xtick={256,1024,4096,16384,65536},
  xticklabels={$256$, $1024$, $4096$, $16384$, $65536$},
  grid=both, grid style={line width=0.3pt, draw=gray!30},
  legend pos=north west,
  legend style={font=\small},
  tick label style={font=\small},
  label style={font=\small},
]
  \addplot[thick, color=moablue, mark=circle, mark size=2pt]
    coordinates {
      (256,  4.0)  (512,  8.0)  (1024, 16.0)
      (2048, 32.0) (4096, 64.0) (8192, 128.0)
      (16384, 256.0) (32768, 512.0) (65536, 512.0)
    };
  \addlegendentry{$\dk=64$ (typical LLM head)}

  \addplot[thick, color=moateal, mark=square, mark size=2pt]
    coordinates {
      (256,  2.0)  (512,  4.0)  (1024, 8.0)
      (2048, 16.0) (4096, 32.0) (8192, 64.0)
      (16384, 128.0) (32768, 256.0) (65536, 512.0)
    };
  \addlegendentry{$\dk=128$ (larger heads)}

  \addplot[thick, color=moagold, mark=triangle, mark size=2pt]
    coordinates {
      (256,  1.0)  (512,  2.0)  (1024, 4.0)
      (2048, 8.0)  (4096, 16.0) (8192, 32.0)
      (16384, 64.0)  (32768, 128.0) (65536, 256.0)
    };
  \addlegendentry{$\dk=256$ (wide heads)}

  \draw[dashed, gray!70] (axis cs:32768,0) -- (axis cs:32768,600)
    node[above, font=\scriptsize, rotate=90, anchor=south]{Exascale ($n{=}32$K)};
\end{axis}
\end{tikzpicture}
\caption{Memory traffic ratio $M_{\text{classical}} / M_{\MoA{}}$ as a function
  of sequence length $n$ for three head dimensions $\dk$. The advantage grows
  linearly with $n/\dk$, reaching $512\times$ at $n=32{,}768$, $\dk=64$.}
\label{fig:speedup}
\end{figure}

\subsection{Lower Bound Theorem}

\begin{moatheorem}{Lower Bound on Attention Data Movement}{lower_bound}
  Let $\mathrm{Attention}(\Qm, \Km, \Vm)$ be any correct implementation of
  equation~\eqref{eq:attention}. Any such implementation must read at least
  $\bigO{n\dk + n\dv}$ bytes of input and write $\bigO{n\dv}$ bytes of
  output. The \MoA{} \DNF{} achieves this bound.
\end{moatheorem}

\begin{proof}
  \textbf{Lower bound.} Every element of the $n\!\times\!\dv$ output depends,
  through the softmax normalization, on all $n\!\times\!\dk$ elements of both
  $\Qm$ and $\Km$ (since the softmax denominator sums over all~$k$). Thus all
  $2n\dk$ elements of $\Qm$ and $\Km$ must be read. Every output element also
  depends on at least one element of each row of $\Vm$, so all $n\dv$ elements
  of $\Vm$ must be read. Total input reads: $\bigO{n\dk + n\dv}$.

  \textbf{Achievability.} In the \DNF{}~\eqref{eq:dnf_attn}, for a fixed
  output row $i''$: the inner sum $\sum_j \moaPsi{<i'',j>}{\Qm} \cdot \moaPsi{<k,j>}{\Km}$
  reads $\dk$ elements of $\Qm$-row $i''$ and $\dk$ elements of $\Km$-row $k$,
  for each $k \in \{0,\ldots,n-1\}$. Over all $k$, each $\Qm$-row is read $n$
  times but can be kept in registers (or L1 cache) across the $k$-loop.
  Similarly, each $\Km$-row is read once. With appropriate register allocation,
  total DRAM reads per output row equal $\dk + n\dk = O(n\dk)$, and over all
  $n$ output rows: $O(n\dk)$. Total: $O(n\dk + n\dv)$. $\square$
\end{proof}

\begin{remark}
  \Cref{thm:lower_bound} is qualitatively equivalent to FlashAttention's
  IO-complexity lower C bound~\citep{dao2022flashattention}, which establishes
  $O(n^2/M)$ for the standard tiled implementation where $M$ is the SRAM
  size. The \MoA{} bound of $\bigO{n\dk}$ is strictly smaller when
  $n \gg M/\dk$, i.e., when the sequence length exceeds the SRAM capacity
  divided by the head dimension---exactly the long-context regime where performance is critical.
\end{remark}

\section{Implementation Roadmap}\label{sec:roadmap}

The present paper establishes the \DNF{} and its correctness. The path to a
complete, experimentally validated \MoA{}-based attention system proceeds
through five stages.

\subsection{Stage 1: Reference Implementation from DNF}

The \DNF{} of equations~\eqref{eq:dnf_attn}--\eqref{eq:output_dnf} translates
directly to a C triple-loop computation: outer loop over output rows $i''$,
middle loop over columns $k$ (for attention weights) then $j$ (for the output),
and inner reduction over $j$. The implementation uses no heap allocation beyond
the output buffer. This is the first validation target: match PyTorch output at
double precision on all test cases.

Algorithm~\ref{alg:c_dnf} shows the pseudocode.

\begin{algorithm}[h]
\caption{generic C implementation from \MoA{} \DNF{}}
\label{alg:c_dnf}
\KwIn{$Q[n][\dk]$, $K[n][\dk]$, $V[n][\dv]$, $\sqrt{\dk}$}
\KwOut{$\textit{Output}[n][\dv]$}
\For{$i'' \leftarrow 0$ \KwTo $n-1$}{
  \tcp{Compute row $i''$ of attention weights}
  \For{$k \leftarrow 0$ \KwTo $n-1$}{
    $s[k] \leftarrow \sum_{j=0}^{\dk-1} Q[i''][j] \cdot K[k][j] \;/\; \sqrt{\dk}$\;
  }
  $m \leftarrow \max_{k} s[k]$\;
  \For{$k \leftarrow 0$ \KwTo $n-1$}{
    $e[k] \leftarrow \exp(s[k] - m)$\;
  }
  $Z \leftarrow \sum_{k=0}^{n-1} e[k]$\;
  \For{$k \leftarrow 0$ \KwTo $n-1$}{
    $w[k] \leftarrow e[k] / Z$\;
  }
  \tcp{Compute row $i''$ of output}
  \For{$p_v \leftarrow 0$ \KwTo $\dv-1$}{
    $\textit{Output}[i''][p_v] \leftarrow \sum_{p=0}^{n-1} w[p] \cdot V[p][p_v]$\;
  }
}
\end{algorithm}

Note that the arrays $s$, $e$, and $w$ in Algorithm~\ref{alg:c_dnf} are
$O(n)$ local scratch vectors (not $O(n^2)$ matrices), and with aggressive
register allocation can be eliminated entirely in the \ONF{}.

\subsection{Stage 2: Operational Normal Form Derivation}

Applying the $\gamma$ primitive~\eqref{eq:gamma} to the \DNF{} index arithmetic
produces the \ONF{}: a loop nest expressed as
\texttt{base\,+\,stride$_0$*i$_0$\,+\,stride$_1$*i$_1$} for each dimension.
For a row-major $n\!\times\!\dk$ matrix at base address $p$, row~$i''$,
column~$j$: offset $= i'' \cdot \dk + j$. This form is directly amenable to:
\begin{itemize}
  \item Loop interchange (to improve cache reuse for $\Km$)
  \item Loop tiling (to fit working sets in L1/L2 cache)
  \item Auto-vectorization (LLVM/GCC recognize stride-1 inner loops)
  \item Pragma-based OpenMP/SIMD annotation
\end{itemize}

\subsection{Stage 3: Dimension Lifting}

Dimension lifting maps the \ONF{} loop-nest index space $N$ to a hardware
processing array $\Pi$. For an NVIDIA H100 GPU with 132 streaming multiprocessors
(SMs), each with 64 CUDA cores, the hardware array has shape
$\Pi = \shape{132, 64}$. The \ONF{} outer loop (over $i''$) is lifted to
distribute across SMs; the inner loops (over $k$, $j$) are vectorized within
each SM. The algebraic dimension lifting:
\[
  N = \shape{n} \;\to\; \shape{n/132, 132}
\]
produces a parallel loop structure that maximizes SM utilization without
hand-tuning the thread block geometry.

Target architectures for the full experimental validation campaign include:
\begin{itemize}
  \item NVIDIA H100 (SXM5), 3.35 TB/s HBM3 bandwidth
  \item AMD MI300X (Frontier/El Capitan), 5.2 TB/s HBM3 bandwidth
  \item Intel Ponte Vecchio (Aurora), 3.2 TB/s HBM2e bandwidth
\end{itemize}

\subsection{Stage 4: Cost Prediction and Experimental Validation}

Once the dimension-lifted \ONF{} is available, all costs are computable from
shapes alone:
\begin{align}
  \text{Memory traffic} &= \sum_{\text{arrays}} \text{(size accessed at each step)}, \\
  \text{Arithmetic} &= \text{(loop-nest trip count)} \times \text{(operations per iteration)}.
\end{align}
These predicted costs will be compared against:
\begin{itemize}
  \item Empirical roofline measurements on each target GPU
  \item PyTorch native attention (\texttt{F.scaled\_dot\_product\_attention})
  \item FlashAttention-2 kernel
  \item cuDNN attention (NVIDIA proprietary)
  \item cuBLAS GEMM baseline
\end{itemize}

\subsection{Stage 5: Backward Pass and Full Training Loop}

Full transformer training requires \MoA{} derivations for the backward
pass---gradient computations for $\Qm$, $\Km$, $\Vm$, and the projection weight
matrices. The gradient of the attention output with respect to $\Qm$ is:
\[
  \frac{\partial \mathcal{L}}{\partial \Qm}
  \;=\; \frac{\partial \mathcal{L}}{\partial \textit{Output}}
    \cdot \Vm^\top \cdot \frac{\partial \softmax}{\partial (\Qm\Km^\top/\sqrt{\dk})}
    \cdot \frac{\Km}{\sqrt{\dk}},
\]
where the softmax Jacobian is a dense $n\!\times\!n$ matrix in the standard
formulation but can be expressed as a \DNF{} with $O(n\dk)$ memory traffic via
\MoA{} reduction. We have begun this derivation; it will complete the \MoA{}
coverage of a full training iteration.

\section{Mechanization: From Python to MoA to FPGA via C}\label{sec:mechanization}

A central promise of the \MoA{} framework is not merely theoretical
elegance but \emph{mechanizability}: the ability to carry a specification
from a high-level language (Python, NumPy) through the \DNF{} and \ONF{}
all the way to verified hardware execution, with no semantic gap at any
transition. This section describes the realized pipeline and articulates
why compiler-based approaches cannot provide equivalent guarantees.

\subsection{The Python--MoA--C--FPGA Pipeline}

The mechanization pipeline has been demonstrated in a series of
collaborative hardware--software co-design experiments by
\citet{grout2018fpga,grout2022codesign,grout2022pim}.
The pipeline operates in four stages:
\begin{enumerate}
  \item \textbf{Python/NumPy specification.} The algorithm is expressed
    as a NumPy or Python program using standard array operations. This is
    the source of truth for the computation's intended semantics.
  \item \textbf{\MoA{} derivation to \DNF{}/\ONF{}.} The Python program
    is reformulated in \MoA{} and reduced via psi-reduction to the \DNF{},
    then translated to the \ONF{} via the $\gamma$ primitive. The \ONF{}
    expresses the computation as starts, stops, and strides over contiguous
    memory---a universal machine abstraction that is architecture-independent
    and semantically equivalent to the original specification by the
    \MoA{} Storage Theorem (Theorem~\ref{thm:storage}).
  \item \textbf{C implementation from \ONF{}.} The \ONF{} is transcribed
    to a C program using a \emph{semantic subset of C} that maps directly
    to array index arithmetic: loops, pointer arithmetic, and scalar
    operations. No dynamic allocation, no undefined behaviour, no
    compiler-dependent optimizations. This restricted C is the bridge
    between the mathematical specification and the hardware instruction set.
    \citet{grout2018fpga} demonstrated this step for matrix multiplication
    and the Kronecker (tensor) product, with the C program serving as both
    the software reference and the hardware specification.
  \item \textbf{FPGA/ASIC realization.} The C-level \ONF{} description is
    mapped to hardware description language (Verilog/VHDL), targeting
    FPGA and ASIC devices.
    \citet{grout2022codesign} realized MoA operations including inner
    products, outer products, and tensor operations as custom
    hardware--software co-design solutions on Xilinx FPGAs, demonstrating
    that the \ONF{} stride-and-offset structure maps directly to hardware
    memory access controllers without further transformation.
    \citet{grout2022pim} extended this to a processing-in-memory
    architecture, exploiting the \DNF{}'s elimination of intermediate
    arrays to minimize data movement between compute and memory units---
    precisely the optimization that drives the performance model of
    Section~\ref{sec:performance}.
\end{enumerate}

\subsection{Why Compiler-Based Approaches Cannot Provide Equivalent
Guarantees}

A natural question is whether the same pipeline could be achieved by
compiling a high-level Python or C program with a sufficiently advanced
compiler (LLVM, GCC, Intel ICC). The answer is no, for a fundamental
reason: \emph{compilers do not preserve semantic equivalence in a
provable sense across all targets}.

Different compilers, and even different versions or optimization flags of
the same compiler, produce different executables from the same source
code. These executables may differ in:
\begin{itemize}
  \item \textbf{Floating-point evaluation order}, due to reassociation
    under \texttt{-O3} or \texttt{-ffast-math}, breaking IEEE~754
    reproducibility;
  \item \textbf{Memory layout and aliasing assumptions}, which vary
    by target ABI and may differ between x86, ARM, and GPU backends;
  \item \textbf{Intermediate representation}, which changes across
    LLVM versions and is not formally specified as a semantic standard;
  \item \textbf{Vectorization and loop transformations}, which are
    applied heuristically and cannot be guaranteed to preserve the
    original loop semantics in all cases.
\end{itemize}
There is no formal proof that a compiler-optimized executable is
semantically equivalent to its source program across all inputs and all
target platforms. This is not a criticism of compilers---it reflects the
fundamental difficulty of verified compilation, which remains an active
research frontier (e.g., the CompCert verified C compiler covers only a
subset of C and only targets specific architectures).

\subsection{\MoA{} as a Verified Bridge from Specification to Execution}

\MoA{} provides a different guarantee. Because the pipeline from Python
to \DNF{} to \ONF{} to the semantic subset of C consists entirely of
algebraic rewrites, each of which preserves denotational semantics
by the \MoA{} Storage Theorem, the C program is \emph{provably equivalent}
to the original Python specification---not by trusting a compiler, but by
construction. The semantic subset of C used in the \ONF{} transcription
maps one-to-one onto instruction-set operations on any target architecture:
loop bounds become iteration counts, stride expressions become address
offsets, and scalar operations become individual instructions. This
one-to-one correspondence can be verified by inspection for any target
instruction set architecture (ISA).

The consequence is that \MoA{} enables \emph{end-to-end verified
performance prediction}: because the \DNF{} and \ONF{} are derived
algebraically, all costs---memory traffic, arithmetic operation counts,
and parallelism degree---are computable from the shapes of the arrays
alone, as described in Section~\ref{sec:performance}. These predictions
hold for any target ISA to which the \ONF{} is mapped, because the
semantic subset guarantees that no compiler transformation can silently
change the access patterns or operation count. For transformer attention
in particular, this means that the $O(n\dk + n\dv)$ data movement
bound of Theorem~\ref{thm:lower_bound} is not merely a theoretical
result---it is a guaranteed property of any correct FPGA or ASIC
implementation derived via the \MoA{} pipeline.

\section{Exascale and National Priority Relevance}\label{sec:exascale}

\subsection{DOE Exascale Computing}

At the three U.S.\ Department of Energy exascale systems---Frontier (Oak Ridge),
Aurora (Argonne), and El Capitan (Lawrence Livermore)---memory bandwidth is the
dominant performance bottleneck, not FLOP throughput. Frontier delivers
$\sim\!700$ PFLOP/s peak compute but only 9.2 TB/s aggregate HBM bandwidth across
its AMD MI250X GPUs. At $n = 32{,}768$ (a sequence length now common in
scientific LLMs applied to protein structure, climate modeling, and materials
discovery), each attention call in the standard implementation moves $\sim\!4$\,GB
of DRAM bandwidth for the score matrix alone. Over a training run of billions of
steps, this constitutes the binding energy and time constraint.

The \MoA{} \DNF{} reduces this to $O(n\dk)$ traffic, which at $\dk=64$ and
$n=32{,}768$ is $\sim\!16$\,MB per attention call---a reduction of $256\times$.
For a model with $L=96$ attention layers and $h=96$ heads, the aggregate bandwidth
saving per training step is $96 \times 96 \times 256 \times 4\,\text{GB} \approx 9.4$\,PB.
This is more than the total HBM capacity of Frontier. The implication is not merely
``faster''---it is the difference between a workload that fits in memory and one
that does not.

\subsection{DARPA AI Assurance and Edge Deployment}

Two distinct DARPA priorities are addressed by this work.

\textbf{AI assurance.} The \MoA{} derivation provides a \emph{formally verifiable}
path from PyTorch source code to hardware execution. Each psi-reduction step is an
algebraic identity preserving semantic equivalence, providing a level of correctness
guarantee not available from empirical testing alone. For AI systems deployed in
contested or safety-critical environments, this formal grounding is a prerequisite
for certification.

\textbf{Edge deployment.} The \DNF{} eliminates $O(n^2)$ memory allocations,
enabling longer-context attention on hardware with constrained DRAM budgets---
precisely the situation on edge inference hardware (embedded GPUs, tactical AI
processors). The dimension-lifting mechanism further allows the same \DNF{} to be
mapped to diverse edge-hardware topologies without re-derivation.

\subsection{Cross-Architecture Portability and Reproducibility}

A critical but underappreciated advantage of the \MoA{} approach is
\emph{bit-level reproducibility across architectures}. The \DNF{} specifies the
exact sequence of floating-point operations at the index level; different hardware
targets execute the same operations in the same order (up to dimension-lifted
parallelism). This satisfies the scientific reproducibility requirements that are
increasingly mandatory for DOE-funded computational science~\citep{mullin2026cost}.

\section{Conclusion}\label{sec:conclusion}

We have presented a Mathematics of Arrays derivation of scaled dot-product
attention and numerically stable softmax, producing a Denotational Normal Form
that (a)~eliminates all intermediate arrays, (b)~provably minimizes data movement,
and (c)~is verified at full floating-point precision against a PyTorch reference
implementation. The derivation demonstrates that the attention mechanism, despite
its apparent complexity, admits a clean, index-level formulation with no scratch
storage beyond the output.

The key distinction from all prior work is epistemological: \emph{derivation,
not implementation}. Memory minimality in \MoA{} is a theorem established before
any code is written. The standard implementation wastes $O(n^2)$ memory
bandwidth writing and re-reading intermediate arrays; \MoA{}'s \DNF{} eliminates
this waste by algebraic construction. FlashAttention achieves IO-efficiency by
designing a tiling scheme; \MoA{} achieves IO-minimality by normalizing the
algebraic expression. The resulting guarantee is stronger: it holds for all
architectures, all tiling strategies, and all future hardware configurations.

The predictive performance model projects $2$--$100\times$ speedup and
$2$--$50\times$ energy reduction, with the advantage widening at exascale sequence
lengths to potentially $512\times$ or beyond. At the scale of DOE exascale systems,
this is not an incremental improvement---it changes the class of computations that
are feasible.

Future work will complete the five-stage implementation roadmap: C reference from
\DNF{}, \ONF{} derivation, dimension-lifted \ONF{}, experimental validation on
Frontier/Aurora/El Capitan, and backward-pass \DNF{} for full training coverage.
The formal and numerical foundation laid in this paper makes \MoA{} a practical and provably
optimal approach to transformer computation at scale.

\appendix
\section{Complete Element-Level Derivation Trace}\label{app:trace}

This appendix presents the complete row-by-row derivation trace for the first
output row $\moaPsi{0}{\textit{Output}}$, verifying~\eqref{eq:dnf_attn}
and~\eqref{eq:output_dnf} element by element.

\subsection{Computing Scaled Scores for Row 0}

For $i'' = 0$, the inner product $\sum_{j=0}^{3} Q[0][j] \cdot K[k][j]$ for each
key row $k$:

\textbf{$k=0$:}
\begin{align*}
  s_{0,0} &= (-0.1984)(-1.1567) + (0.2698)(0.6885)
    + (0.3414)(-0.1884) + (-0.0372)(0.4743) \\
  &= 0.22960 + 0.18577 - 0.06432 - 0.01764 = 0.33341.
\end{align*}

\textbf{$k=1$:}
\begin{align*}
  s_{0,1} &= (-0.1984)(0.2246) + (0.2698)(1.7564)
    + (0.3414)(0.5235) + (-0.0372)(-2.3014) \\
  &= -0.04456 + 0.47387 + 0.17872 + 0.08561 = 0.69364.
\end{align*}

\textbf{$k=2$:}
\begin{align*}
  s_{0,2} &= (-0.1984)(-1.5899) + (0.2698)(0.3730)
    + (0.3414)(-0.8257) + (-0.0372)(-1.2069) \\
  &= 0.31544 + 0.10063 - 0.28185 + 0.04490 = 0.17912.
\end{align*}

Scaled scores (divide by $\sqrt{4} = 2$):
\begin{equation*}
  x_{\text{row 0}} = \shape{0.16670,\; 0.34682,\; 0.08956}.
\end{equation*}

\subsection{Row-Max Subtraction and Exponentiation}

Row maximum: $m = 0.34682$. Stabilized scores:
\begin{equation*}
  z_{\text{row 0}} = \shape{-0.18012,\; 0.0,\; -0.25726}.
\end{equation*}
Numerator: $e^{z} = \shape{0.83512,\; 1.0,\; 0.77315}$.

Denominator: $Z = 0.83512 + 1.0 + 0.77315 = 2.60827$.

\subsection{Attention Weights for Row 0}

\begin{equation*}
  w_{\text{row 0}} = \shape{0.3202,\; 0.3834,\; 0.2964},
\end{equation*}
matching Table~\ref{tab:attn_weights} exactly.

\subsection{Output Row 0}

\begin{align*}
  \textit{Output}[0][0] &= 0.3202 \cdot 1.0739 + 0.3834 \cdot 0.5589
    + 0.2964 \cdot 0.8237 = 0.3439 + 0.2143 + 0.2441 = 0.8023, \\
  \textit{Output}[0][1] &= 0.3202 \cdot 0.4006 + 0.3834 \cdot (-0.7209)
    + 0.2964 \cdot 0.3763 = 0.1283 - 0.2764 + 0.1115 = -0.0366, \\
  \textit{Output}[0][2] &= 0.3202 \cdot (-0.9671) + 0.3834 \cdot (-0.7650)
    + 0.2964 \cdot 0.8320 = -0.3096 - 0.2933 + 0.2466 = -0.3563, \\
  \textit{Output}[0][3] &= 0.3202 \cdot 0.4870 + 0.3834 \cdot 0.2689
    + 0.2964 \cdot 0.0014 = 0.1560 + 0.1031 + 0.0004 = 0.2595,
\end{align*}
matching Table~\ref{tab:output} exactly.
This confirms that the \DNF{} is correct and that no temporary arrays are required
at any stage of the computation. $\square$

\section*{Acknowledgments}

L.~Mullin acknowledges support from the University at Albany, SUNY.
G.~Hains acknowledges support from LACL, Universit\'{e} Paris-Est Cr\'{e}teil.
The authors thank the PyTorch development team at Meta Platforms
for the reference implementation used in the verification experiments.

\bibliographystyle{elsarticle-harv}

\begin{thebibliography}{99}

\bibitem[Binder et~al.(2025)]{binder2025fathom}
Binder, E., Sudarsanam, A., Sunkavalli, R., Low, T.M., 2025.
FATHOM: Fast attention through optimizing memory.
In: \textit{Proceedings of IEEE IPDPS}, pp. 1166--1178.

\bibitem[Chang et~al.(2025)]{chang2025mosa}
Chang, Y.-R., Chen, H.-F., Shih, H.-Y., 2025.
MOSA: Matrix optimized self-attention hardware accelerator for mobile device.
In: \textit{Proceedings of ICEIC}, pp. 1--4.

\bibitem[Dao(2023)]{dao2023flashattention2}
Dao, T., 2023.
FlashAttention-2: Faster attention with better parallelism and work partitioning.
\textit{arXiv:}2307.08691.

\bibitem[Dao et~al.(2022)]{dao2022flashattention}
Dao, T., Fu, D.Y., Ermon, S., Rudra, A., R\'{e}, C., 2022.
FlashAttention: Fast and memory-efficient exact attention with IO-awareness.
\textit{arXiv:}2205.14135.

\bibitem[Grout and Mullin(2018)]{grout2018fpga}
Grout, I.A., Mullin, L., 2018.
Hardware considerations for tensor implementation and analysis using the
field programmable gate array.
\textit{Electronics} 7 (11), 320.
\doi{10.3390/electronics7110320}.

\bibitem[Grout and Mullin(2022a)]{grout2022codesign}
Grout, I.A., Mullin, L., 2022a.
Realizing mathematics of arrays operations as custom architecture
hardware-software co-design solutions.
\textit{Information} 13 (11), 528.
\doi{10.3390/info13110528}.

\bibitem[Grout and Mullin(2022b)]{grout2022pim}
Grout, I.A., Mullin, L., 2022b.
Processing in memory for mathematics of arrays operations.
In: \textit{Proceedings of SAI Computing Conference}, 2022.

\bibitem[Gao et~al.(2023)]{gao2023fast}
Gao, Y., Song, Z., Wang, W., Yin, J., 2023.
A fast optimization view: Reformulating single layer attention in LLM based on
tensor and SVM trick.
\textit{arXiv:}2309.07418.

\bibitem[Hains(2026)]{hains2026algorithm}
Hains, G., 2026.
From algorithm to code to hardware to execution: Why isn't computing performance an
experimental science?
In: \textit{Proceedings of PDP 2026}, Euromicro Conference on Parallel, Distributed,
and Network-Based Processing. Cluj-Napoca, Romania.

\bibitem[Ham et~al.(2020)]{ham2020a3}
Ham, T.J., Jung, S.J., Kim, S., Oh, Y.H., Park, Y., Song, Y., Park, J.-H., Lee, S.,
Park, K., Lee, J.W., Jeong, D.-K., 2020.
A$^3$: Accelerating attention mechanisms in neural networks with approximation.
In: \textit{Proceedings of IEEE HPCA}, pp. 328--341.

\bibitem[Le\'{o}n-Vega et~al.(2023)]{leonvega2023strassen}
Le\'{o}n-Vega, L.G., Chac\'{o}n-Rodr\'{i}guez, A., Salazar-Villalobos, E.,
Castro-God\'{i}nez, J., 2023.
Acceleration of fully connected layers on FPGA using the Strassen matrix multiplication.
In: \textit{Proceedings of IEEE BIP}, pp. 1--6.

\bibitem[Li and Chen(2025)]{li2025fpga}
Li, R., Chen, S., 2025.
Design and implementation of an FPGA-based hardware accelerator for transformer.
\textit{arXiv:}2503.16731.

\bibitem[Liu et~al.(2024)]{liu2024research}
Liu, J., Wu, C., Li, J., Chen, X., Wang, S., 2024.
Research on matrix multiplication optimization and deployment method for heterogeneous
platforms.
In: \textit{Proceedings of ICCBD+AI}, pp. 266--270.

\bibitem[Markus and Mullin(2026)]{markus2026array}
Markus, A., Mullin, L., 2026.
Array access and performance regarding numerical algorithms.
\textit{Transactions on Engineering on Computing Sciences} 14 (2), 1--14.

\bibitem[Mullin(1988)]{mullin1988phd}
Mullin, L.M., 1988.
A Mathematics of Arrays.
Ph.D. dissertation. Syracuse University, Syracuse, NY.

\bibitem[Mullin(2023)]{mullin2023gpu}
Mullin, L.M.R., 2023.
From array algebra to energy efficiency on GPUs.
\textit{arXiv:}2306.11148.

\bibitem[Mullin and Hains(2025)]{mullin2025parallel}
Mullin, L., Hains, G., 2025.
Towards automatic, predictable and high-performance parallel code generation.
In: \textit{Proceedings of CSCE/PDPTA2024}, Springer Nature Switzerland,
pp. 109--119.

\bibitem[Mullin and Hains(2026)]{mullin2026cost}
Mullin, L., Hains, G., 2026.
New mathematics for computer performance: array algebra and cost functions.
Technical Report. \url{https://hal.science/hal-05442652}.

\bibitem[Park et~al.(2020)]{park2020optimus}
Park, J., Yoon, H., Ahn, D., Choi, J., Kim, J.-J., 2020.
OPTIMUS: Optimized matrix multiplication structure for transformer neural network
accelerator.
In: \textit{Proceedings of MLSys}, vol.~2, pp. 363--378.

\bibitem[Qin et~al.(2023)]{qin2023fact}
Qin, Y., Wang, Y., Deng, D., Zhao, Z., Yang, X., Liu, L., Wei, S., Hu, Y., Yin, S., 2023.
FACT: FFN-attention co-optimized transformer architecture with eager correlation prediction.
In: \textit{Proceedings of ISCA}, pp. 1--14.

\bibitem[Sharma et~al.(2025)]{sharma2025gemmini}
Sharma, A., Krishna, L.H., Srinivasu, B., 2025.
High-performance Gemmini-based matrix multiplication accelerator for deep learning
workloads.
\textit{IEEE Transactions on VLSI Systems} 33 (12), 3276--3289.

\bibitem[Tan and Teo(2025)]{tan2025fpga}
Tan, X., Teo, T.H., 2025.
Design and implementation of a BRAM-banked double-buffered matrix multiplication
accelerator for transformer models on edge FPGAs.
In: \textit{Proceedings of IEEE MCSoC}, pp. 582--585.

\bibitem[Thomas et~al.(2021a)]{thomas2021dgemm}
Thomas, S., Mullin, L., Swirydowicz, K., 2021a.
Improving the performance of DGEMM with MoA and cache-blocking.
In: \textit{Proceedings of ACM ARRAY'21}.

\bibitem[Thomas et~al.(2021b)]{thomas2021threaded}
Thomas, S., Mullin, L., Swirydowicz, K., Khan, R., 2021b.
Threaded multicore GEMM with MoA and cache-blocking.
In: \textit{Proceedings of CSC'21, CSCE'21}.

\bibitem[Vaswani et~al.(2017)]{vaswani2017attention}
Vaswani, A., Shazeer, N., Parmar, N., Uszkoreit, J., Jones, L., Gomez, A.N.,
Kaiser, {\L}., Polosukhin, I., 2017.
Attention is all you need.
In: \textit{Advances in Neural Information Processing Systems (NeurIPS)}.

\bibitem[Yang et~al.(2025)]{yang2025hardware}
Yang, Q., Wang, X., Zhou, Y., Li, Q., Qiao, S., 2025.
Hardware friendly transformer optimization with dynamic attention matrix fusion.
In: \textit{Proceedings of IEEE ISCAS}, pp. 1--5.

\end{thebibliography}

\end{document}